\documentclass[11pt]{article}

\usepackage[preprint]{latex/acl}

\usepackage{times}
\usepackage{latexsym}

\usepackage[T1]{fontenc}

\usepackage[utf8]{inputenc}

\usepackage{microtype}

\usepackage{graphicx}

\usepackage{graphicx}
\usepackage{subcaption}
\usepackage{float}

\usepackage{booktabs}
\usepackage{array}
\usepackage{multirow}
\usepackage{tabularx}
\usepackage{threeparttable}

\usepackage{xcolor}
\usepackage{colortbl}

\usepackage{amsmath}
\usepackage{amssymb}
\usepackage{amsfonts}
\usepackage{mathtools}
\mathtoolsset{showonlyrefs}
\usepackage{amsthm}
\usepackage{bm}
\usepackage{bbm}
\usepackage{dsfont}

\usepackage{algorithm}
\usepackage{algorithmic}

\usepackage{enumitem}
\usepackage{enumerate}

\usepackage{url}
\usepackage{multicol}
\usepackage{indentfirst}
\usepackage{tikz}
\usepackage{bbding}
\usepackage{fontawesome}
\usepackage[normalem]{ulem}
\useunder{\uline}{\ul}{}
\usepackage{lipsum}
\usepackage{natbib}
\usepackage{appendix}
\usepackage{import}

\usepackage[skins,breakable]{tcolorbox}

\usepackage{hyperref}
\usepackage{cleveref}

\definecolor{lossGreen}{RGB}{0, 130, 60}
\definecolor{gainRed}{RGB}{200, 30, 30}

\newcommand{\impgain}[1]{
  \textcolor{gainRed}{$\uparrow$\,#1\%}
}

\newcommand{\imploss}[1]{
  \textcolor{lossGreen}{$\downarrow$\,#1\%}
}

\theoremstyle{plain}
\newtheorem{theorem}{Theorem}[section]

\newtheorem{corollary}[theorem]{Corollary}
\theoremstyle{definition}

\theoremstyle{remark}
\newtheorem*{remark}{Remark}

\crefname{theorem}{Theorem}{Theorems}
\crefname{proposition}{Proposition}{Propositions}
\crefname{lemma}{Lemma}{Lemmas}
\crefname{corollary}{Corollary}{Corollaries}
\crefname{definition}{Definition}{Definitions}
\crefname{assumption}{Assumption}{Assumptions}
\crefname{remark}{Remark}{Remarks}

\crefname{equation}{equation}{equations}
\Crefname{equation}{Equation}{Equations}

\newcommand\blfootnote[1]{
  \begingroup
  \renewcommand\thefootnote{}\footnote{#1}
  \addtocounter{footnote}{-1}
  \endgroup
}

\useunder{\uline}{\ul}{}
\definecolor{skyblue}{RGB}{0,120,215}

\newcounter{mycounter} 

\newcommand{\findingbox}[2][]{
    \refstepcounter{mycounter} 
    \if\relax\detokenize{#1}\relax\else\label{#1}\fi 
    \begin{tcolorbox}[colframe=black,
                      arc=1pt,
                      boxsep=-2pt,
                      before skip=5pt,  
                      after skip=5pt,   
                      ]
        \noindent{\textbf{\textit{Finding \themycounter.}}} #2 
    \end{tcolorbox}
}

\title{Learning from Hard Prompts: Difficulty-aware Advantage \\
Amplification in Dynamic Sampling}

\author{
Siyuan Gan$^{1}$,
Yuhan Li$^{1,2}$,
Xiran Wang$^{1,2}$,
Linjian Meng$^{\dag 2}$,
Boyan Wang$^{\dag 1}$,\\
\textbf{Zhen Zhao$^{2}$,
Jing Huo$^{1}$,
Lei Bai$^{2}$,
Yang Gao$^{1}$} \\
$^1$State Key Laboratory of Novel Software Technology, Nanjing University, Nanjing, China \\
$^2$Shanghai Artificial Intelligence Laboratory, Shanghai, China\\
gansiyuan@smail.nju.edu.cn,
menglinjian@pjlab.org.cn,
boyanwang@nju.edu.cn
}

\makeatletter
\def\blfootnote#1{
  \begin{NoHyper}
  \xdef\@thefnmark{}
  \@footnotetext{#1}
  \end{NoHyper}
}
\makeatother

\begin{document}
\maketitle
\blfootnote{$\dag$ Corresponding author.}
\begin{abstract}
Decoupled Clip and Dynamic Sampling Policy Optimization (DAPO) is a prominent variant of Group Relative Policy Optimization (GRPO). DAPO introduces several improvements over GRPO. Among these, \textit{Dynamic Sampling} contributes the most to DAPO's accuracy gains relative to GRPO. 
To improve accuracy, Dynamic Sampling enhances training stability by eliminating zero policy gradients from zero advantages. Specifically, it avoids such zero gradients by filtering out prompts where sampled responses are either entirely correct or incorrect.
However, our theoretical analysis shows that Dynamic Sampling decrease training efficiency as it cannot effectively utilize hard-to-sample correct responses on hard prompts. 
Formally, it asymmetrically amplifies the advantages of distinct responses to the same prompts. On hard prompts, incorrect responses undergo greater amplification than correct ones. This leads the model to avoid generating the observed incorrect responses rather than capitalizing on the hard-to-sample correct ones on hard prompts, resulting in low training efficiency.
To improve training efficiency, we propose \textit{Direct Advantage Amplification} (DAA), which amplifies the advantages of hard-to-sample correct responses on hard prompts, as obtained by Dynamic Sampling. This ensures that, when Dynamic Sampling is used, these hard-to-sample responses can be effectively capitalized on, implying higher training efficiency. 
By integrating DAA into DAPO, we obtain \textit{Difficulty-aware Advantage Amplification Policy Optimization} (DA3PO), which is implemented with fewer than 30 lines of code from DAPO. 
Experiments show that DA3PO significantly outperforms GRPO and other classical GRPO variants.

\end{abstract}

\section{Introduction}\label{sec:Introduction}
The release of DeepSeek-R1~\citep{deepseek-r1} has drawn significant attention to Reinforcement Learning from Verifiable Rewards (RLVR), establishing Group Relative Policy Optimization (GRPO)~\citep{grpo} as a prevailing approach for training Large Language Models (LLMs). Then, to enhance the accuracies across a variety of tasks, various GRPO variants have been proposed, such as Decoupled Clip and Dynamic Sampling Policy Optimization (DAPO)~\citep{dapo} and Group Sequence Policy Optimization (GSPO)~\citep{gspo}, which have further extended GRPO. Among them, DAPO is one of the most popular variants~\citep{liu2025prorl,yu2026knowrl,yuchen2026predictability,zhang2026patho,fipo}.

Compared to GRPO, DAPO introduces several significant improvements. Among these, Dynamic Sampling contributes the most to DAPO's accuracy enhancements from GRPO, as presented by the experimental results of the original paper of DAPO~\citep{dapo}. To improve the accuracy, Dynamic Sampling enhances training stability, a critical factor for achieving accuracy enhancements, by eliminating zero policy gradients caused by zero advantages~\citep{dapo}.
Specifically, such zero policy gradients come from prompts where sampled responses are entirely correct or incorrect. Therefore, to avoid producing such zero policy gradients, Dynamic Sampling over-samples prompts and filters out prompts where sampled responses are entirely correct or incorrect.

Unfortunately, while Dynamic Sampling improves training stability, our theoretical analysis indicates that it will reduce training efficiency, another crucial aspect for attaining accuracy enhancements. Specifically, the decline in training efficiency primarily stems from Dynamic Sampling's inability to effectively capitalize on hard-to-sample correct responses on hard prompts. More precisely, we demonstrate that Dynamic Sampling will asymmetrically amplify the advantages of different responses to the same prompts, with this amplification varying based on prompt difficulty. On hard prompts, incorrect responses undergo greater amplification relative to correct ones. This asymmetric amplification causes the model to avoid generating observed incorrect responses, rather than capitalizing on the correct ones on hard prompts, thereby reducing training efficiency as the correct responses on hard prompts are exceedingly hard to sample.

To improve training efficiency, we propose \textit{Direct Advantage Amplification} (DAA). Its key insight is to amplify the advantages of hard-to-sample correct responses on hard prompts, as obtained by Dynamic Sampling. Formally, after obtaining such advantages via Dynamic Sampling, we multiply these advantages by a positive constant greater than one to amplify them. This ensures that, when Dynamic Sampling is used, these hard-to-sample correct responses on hard prompts can be effectively capitalized on, implying higher training efficiency. Then, by integrating DAA into DAPO, we obtain \textit{Difficulty-aware Advantage Amplification Policy Optimization} (DA3PO). More surprisingly, since the implementation of DAA is straightforward, DA3PO requires fewer than 30 lines of code derived from DAPO (see more details in \Cref{sec:method}).

To evaluate our DA3PO, we conduct experiments on seven mathematical reasoning benchmarks: AIME24, AIME25, AIME26, AMC, Minerva, Olympiad and MATH. 

Experimental results demonstrate that DA3PO yields consistent accuracy enhancements across all tested models and benchmarks, significantly outperforming GRPO and classical GRPO variants DAPO and GSPO.

\section{Preliminary}\label{sec:Preliminary}

\subsection{Group Relative Policy Optimization (GRPO)}

GRPO~\cite{grpo} computes advantages purely from group-level rewards, making it an effective approach for optimizing a policy $\pi_\theta$. At each training step of GRPO training, a prompt $x$ is sampled from the dataset $\mathcal{D}$ and the policy $\pi_\theta$ then generates a group of $G$ responses $\{o_1, \dots, o_G\}$, where the conditional probability of response $o_i$ is:
\begin{equation}
\pi_\theta(o_i \mid x) = \prod_{t=1}^{|o_i|} \pi_\theta(o_{i,t} \mid x,\, o_{i,<t}).
\end{equation}
Let $\pi_{\theta_{\mathrm{old}}}$ denote the behavior policy before each gradient update. The importance sampling ratio at token position $t$ of the $i$-th response is:
\begin{equation}
\label{eq:importance_ratio}
r_{i,t}(\theta) = \frac{\pi_\theta(o_{i,t} \mid x,\, o_{i,<t})}{\pi_{\theta_{\mathrm{old}}}(o_{i,t} \mid x,\, o_{i,<t})}.
\end{equation}
A rule-based verifier assigns a reward $R_i$ which is defined as $R_i = 1$ if correct, and $R_i = -1$ otherwise. The advantage for the $i$-th response at token position $t$ is computed:
\begin{equation}
\label{eq:grpo_advantage}
\hat{A}_{i,t} = \frac{R_i - \mathrm{mean}(\{R_j\}_{j=1}^G)}{\mathrm{std}(\{R_j\}_{j=1}^G)}.
\end{equation}
The policy is updated by maximizing a clipped objective with per-sample loss averaging:
\begin{equation}
\label{eq:grpo_objective}
\begin{split}
    \mathcal{J}_{\mathrm{GRPO}}(\theta) = \mathbb{E}_{x \sim \mathcal{D},\; \{o_i\}_{i=1}^G \sim \pi_{\theta_{\mathrm{old}}}(\cdot \mid x)}
    \Bigg[ \frac{1}{G} \\
    \sum_{i=1}^{G} \frac{1}{|o_i|} \sum_{t=1}^{|o_i|} \mathcal{L}_{i,t}(\theta) \Bigg],
\end{split}
\end{equation}
where the per-token clipped loss is:
\begin{equation}
\label{eq:grpo_loss}
\begin{split}
    \mathcal{L}_{i,t}(\theta) = \min\!\Big(
    r_{i,t}(\theta)\,\hat{A}_{i,t},\;
    \mathrm{clip} (r_{i,t}(\theta),\, \\
    1-\epsilon,\, 1+\epsilon)\,\hat{A}_{i,t}
    \Big),
\end{split}
\end{equation}
and $\epsilon$ is the symmetric clipping range.

\subsection{Decoupled Clip and Dynamic Sampling Policy Optimization (DAPO)}
\label{sec:preliminary_dapo}

Compared to GRPO, DAPO~\cite{dapo} introduces several improvements to enhance training efficiency and training stability, resulting in higher accuracy than GRPO. 

\paragraph{Clip-Higher.} 
Firstly, to mitigate training inefficiency arising from low exploration, DAPO applies Clip-Higher to increase the probability of low-probability exploration tokens. Formally, Clip-Higher decouples the symmetric clipping range in Eq.~\eqref{eq:grpo_loss} into asymmetric bounds: 
\begin{equation} 
\label{eq:dapo_clip} 
\mathrm{clip}(r_{i,t}(\theta),\, 1-\epsilon_l,\, 1+\epsilon_h), \epsilon_h > \epsilon_l.
\end{equation}

\paragraph{Overlong Reward Shaping.} 
Then, to mitigate low training efficiency prolonged by sampling of overly long responses, DAPO adds an overlong reward to the GRPO correctness reward to avoid sampling too long responses: 
\begin{equation} 
\label{eq:dapo_reward} 
R_i = R_{\mathrm{acc}}(o_i) + R_{\mathrm{len}}(o_i), 
\end{equation}
the overlong reward shaping is: 
\begin{equation}
\label{eq:dapo_length_reward}
\small 
R_{\mathrm{len}}(o_i) = \begin{cases} 0, & |o_i| \leq L_{m} - L_{c} \\[4pt] \dfrac{(L_{m} - L_{c}) - |o_i|}{L_{c}}, & L_{m} - L_{c} < |o_i| \leq L_{m} \\[4pt] -1, & |o_i| > L_{m}.
\end{cases} 
\end{equation} 
Here $L_m$ is the maximum generation length and $L_c$ is the soft punish cache. 

\paragraph{Token-Level Loss.} 
In addition, to improve training stability, DAPO replaces the per-sample normalizer of GRPO with a single global token-level normalizer in Eq.~\eqref{eq:dapo_objective}:
\begin{equation}
\mathbb{E}_{x \sim \mathcal{D},\; \{o_i\}_{i=1}^G \sim \pi_{\theta_{\mathrm{old}}}(\cdot \mid x)}
    \Bigg[ \frac{1}{\sum_{i=1}^{G}|o_i|} \\
    \sum_{i=1}^{G} \sum_{t=1}^{|o_i|} \mathcal{L}_{i,t}(\theta) \Bigg].
\end{equation}
This formulation ensures that every token contributes equally to the gradient, thereby preventing an unhealthy increase in entropy~\cite{dapo}, which would otherwise degrade training stability.

\paragraph{Dynamic Sampling.} 
Lastly, to address low training stability caused by zero policy gradients, DAPO utilizes Dynamic Sampling. Specifically, when a group of sampled responses $\{o_i\}_{i=1}^G$ of a prompt receive the same reward, the advantage is zero, which yields zero policy gradients, shrinking the magnitude and increasing the noise sensitivity of the batch gradient, thereby degrading sample utilization~\cite{dapo}. To address this problem, Dynamic Sampling ensures that for all training prompts, the sampled responses $\{o_i\}_{i=1}^G$ have distinct rewards. Formally, we define $p = \mathbb{E}_{o \sim \pi_\theta(\cdot \mid x)} [\mathbb{I}[R_{\mathrm{acc}}(o) > 0]]$ for the probability that a response sampled from $\pi_\theta(\cdot \mid x)$ is correct, and define the empirical positive-response rate within a group as

$\hat{p} = \frac{1}{G}\sum_{i=1}^{G}\mathbb{I}[R_{\mathrm{acc}}(o_i) > 0]$, 

where $\mathbb{I}[\cdot]$ is the indicator function. Dynamic Sampling will over-sample and filter out prompts, leaving all prompts in the batch with $0 < \hat{p} < 1$ and keeping a consistent number of prompts. 

Integrating all four improvements, the DAPO objective takes the form: 
\begin{equation}
\label{eq:dapo_objective}
\small
\begin{split}
    \mathcal{J}_{\mathrm{DAPO}}(\theta) = \mathbb{E}_{x \sim \mathcal{D},\; \{o_i\}_{i=1}^G \sim \pi_{\theta_{\mathrm{old}}}(\cdot \mid x)}
    \Bigg[ \frac{1}{\sum_{i=1}^{G}|o_i|} \\
    \sum_{i=1}^{G} \sum_{t=1}^{|o_i|} \mathcal{L}_{i,t}(\theta) \Bigg],
    \quad \text{s.t.} \quad 0 < \hat{p} < 1,
\end{split}
\end{equation}
where the per-token clipped surrogate loss is:
\begin{equation}
\label{eq:dapo_loss}
\begin{split}
    \mathcal{L}_{i,t}(\theta) = \min\!\Big(
    r_{i,t}(\theta)\,\hat{A}_{i,t},\;
    \mathrm{clip}(r_{i,t}(\theta),\, \\
    1-\epsilon_l,\, 1+\epsilon_h)\,\hat{A}_{i,t}
    \Big),
\end{split}
\end{equation}
and the advantage estimate $\hat{A}_{i,t}$ retains the same definition as in Eq.~\eqref{eq:grpo_advantage}, computed with the DAPO reward $R_i$ from Eq.~\eqref{eq:dapo_reward}.

\section{Theory}\label{sec:theory}
\begin{figure*}[!t]
    \centering
    \begin{minipage}[t]{0.49\textwidth}
        \centering
        \includegraphics[width=\linewidth]{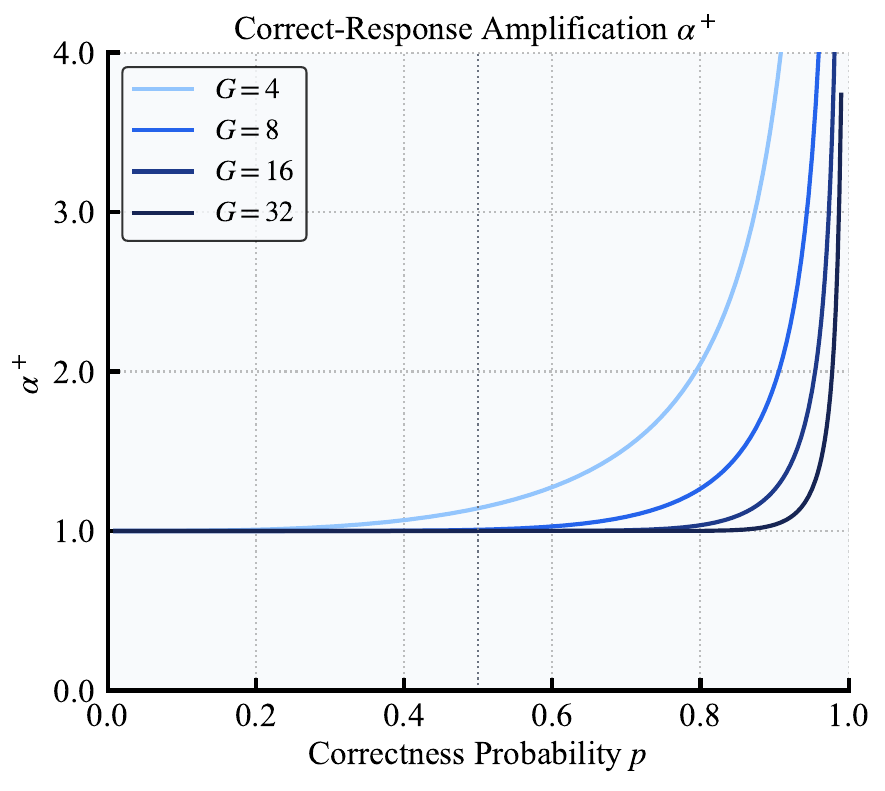}
        \subcaption{$\alpha^{+}$ as a function of $p$ for $G\in\{4,8,16,32\}$.}
        \label{fig:amp_A}
    \end{minipage}
    \hfill
    \begin{minipage}[t]{0.49\textwidth}
        \centering
        \includegraphics[width=\linewidth]{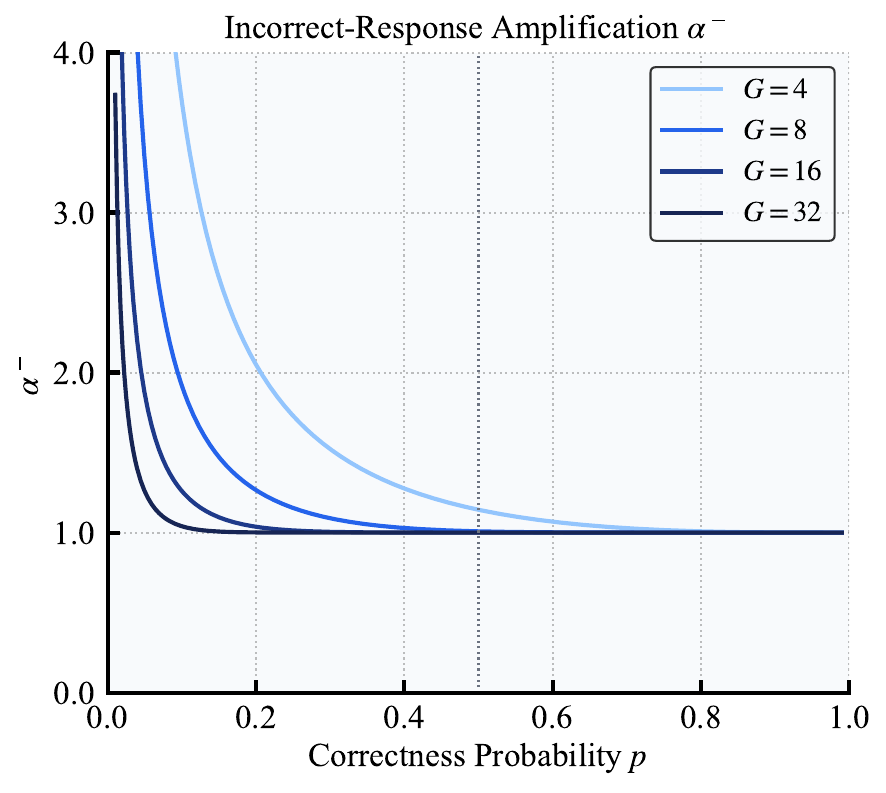}
        \subcaption{$\alpha^{-}$ as a function of $p$ for $G\in\{4,8,16,32\}$.}
        \label{fig:amp_B}
    \end{minipage}
    \caption{
        Asymmetric amplification factors $\alpha^{+}$ and $\alpha^{-}$ under Dynamic Sampling
        (Theorem~\ref{thm:asym_amp}).
    }
    \label{fig:amplification}
\vspace{-0.1cm}
\end{figure*}

\begin{figure}[!t]
    \centering
    \begin{minipage}[t]{0.49\textwidth}
        \centering
        \includegraphics[width=\linewidth]{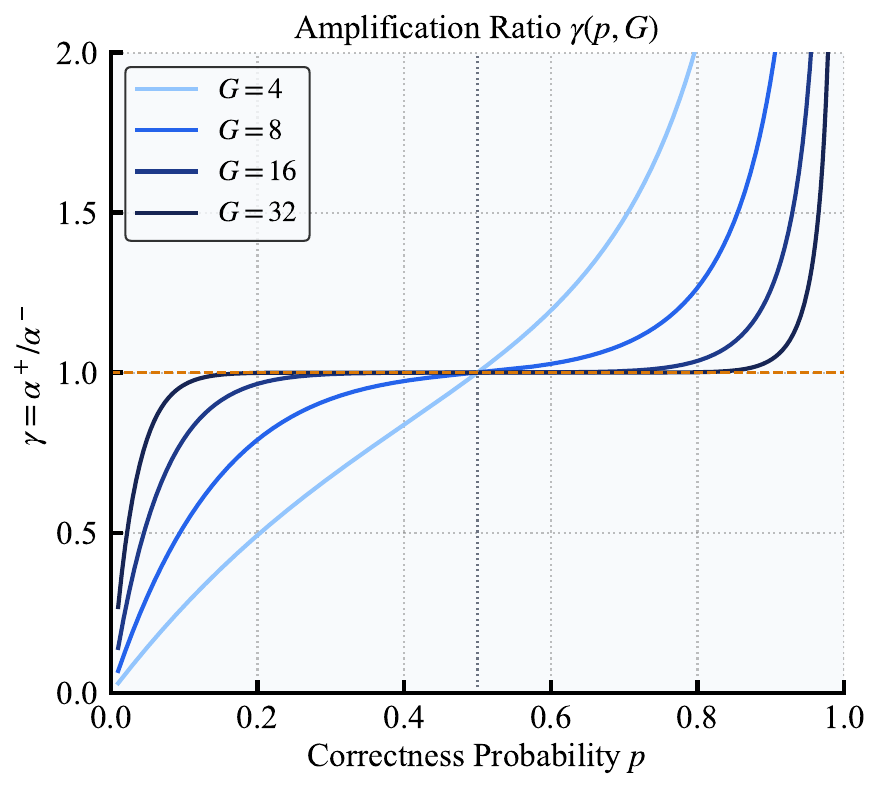}
        
        \caption{
            Ratio $\gamma=\alpha^{+}/\alpha^{-}$ satisfies $\gamma<1$ across the entire shaded region ($p<0.5$), regardless of group size $G$
        (Corollary~\ref{cor:hard_prompts}).
        }
    \label{fig:amp_C}
    \end{minipage}
\vspace{-0.2cm}
\end{figure}
As presented in the experimental results of the original DAPO paper, among the four improvements, Dynamic Sampling contributes the most to DAPO's accuracy gain over GRPO. However, in this section, we demonstrate that while Dynamic Sampling enhances accuracy by improving training stability, it reduces training efficiency. Specifically, Dynamic Sampling induces an asymmetric amplification of advantages (\cref{thm:asym_amp}), leading the model to avoid generating observed incorrect responses rather than strengthening the hard-to-sample correct ones on hard prompts (Corollary \ref{cor:hard_prompts}), resulting in low training efficiency.

Now, we demonstrate the details of the asymmetric amplification. Precisely, after applying Dynamic Sampling, the expected advantages of the correct and incorrect responses are multiplied by the coefficients $\alpha^{+}$ and $\alpha^{-}$, respectively.

\begin{theorem}[Proof is in Appendix~\ref{app:proof_thm1}]
\label{thm:asym_amp}
For any prompt $x$ with $p\in(0,1)$ and group size $G\ge2$, under $0 < \hat{p} < 1$:

\begin{equation}\label{eq:amp_correct}
\begin{split}
&\mathbb{E}[\hat{A}_i \mid 0 < \hat{p} < 1,\, R_{\mathrm{acc}}(o_i){=}1] \\
&\quad= \alpha^{+}\cdot\mathbb{E}[\hat{A}_i \mid R_{\mathrm{acc}}(o_i){=}1],
\end{split}
\end{equation}
\begin{equation}\label{eq:amp_incorrect}
\begin{split}
&\mathbb{E}[\hat{A}_i \mid 0 < \hat{p} < 1,\, R_{\mathrm{acc}}(o_i){=}{-1}] \\
&\quad= \alpha^{-}\cdot\mathbb{E}[\hat{A}_i \mid R_{\mathrm{acc}}(o_i){=}{-1}].
\end{split}
\end{equation}
where the amplification factors are
\begin{equation*}
\alpha^{+} := \frac{1}{1-p^{G-1}}, \qquad \alpha^{-} := \frac{1}{1-(1-p)^{G-1}}.
\end{equation*}
\end{theorem}

\begin{remark}
    The factors $\alpha^{+}$ and $\alpha^{-}$ amplify the conditional expected advantage for correct and incorrect responses, respectively. The amplification factors reveal a striking asymmetry as prompt difficulty varies. \Cref{fig:amplification} visualize two amplification factors across the full $(p,G)$ parameter space. As shown in \Cref{fig:amp_A} and \Cref{fig:amp_B}, $\alpha^{+}$ remains close to~$1$ throughout the hard-prompt regime ($p\ll 0.5$) while $\alpha^{-}$ grows large precisely there, indicating that Dynamic Sampling disproportionately amplifies the penalty signal for incorrect responses. 
\end{remark}

\begin{corollary}
\label{cor:hard_prompts}
Let $\gamma(p,G):=\alpha^{+}/\alpha^{-}$. Then:
$$\gamma(p,G)<1,\quad\text{if }p<0.5;$$
$$\gamma(p,G)>1,\quad\text{if }p>0.5;$$
$$\gamma(p,G)=1,\quad\text{if and only if }p=0.5.$$
\end{corollary}
This corollary follows directly from \cref{thm:asym_amp}: for $p<0.5$ we have $p<1-p$, so $p^{G-1}<(1-p)^{G-1}$, giving $\alpha^{+}<\alpha^{-}$ and hence $\gamma<1$. Symmetrically, for $p>0.5$ we have $p>1-p$, giving $\alpha^{+}>\alpha^{-}$ and hence $\gamma>1$. The equality $\gamma=1$ holds if and only if $p=1-p=0.5$.

\begin{remark}
\label{rem:numerical}
\Cref{fig:amp_C} confirms Corollary~\ref{cor:hard_prompts}: the ratio $\gamma=\alpha^{+}/\alpha^{-}$ falls below~$1$ for all $p<0.5$ and all group sizes.
Specially, for $G=8$ and $p=0.1$: $\alpha^{+}\!\approx\!1.00$, $\alpha^{-}\!\approx\!1.92$, $\gamma\!\approx\!0.52$. For $G=16$ and $p=0.1$: $\alpha^{+}\!\approx\!1.00$, $\alpha^{-}\!\approx\!1.26$, $\gamma\!\approx\!0.79$. In both cases the penalty on incorrect responses is amplified substantially more than the reward for correct ones, and a smaller group size intensifies this asymmetry. 
\end{remark}

Corollary~\ref{cor:hard_prompts} reveals that for hard prompts ($p < 0.5$), the ratio $\gamma = \alpha^{+}/\alpha^{-}$ is strictly less than~$1$. As $p \to 0$, this asymmetry becomes extreme: $\alpha^{-}$ grows unboundedly, while $\alpha^{+}$ remains close to~$1$, implying Dynamic Sampling asymmetrically amplifies the advantage of incorrect responses more than that of correct ones. Consequently, rather than capitalizing on the hard-to-sample correct responses on hard prompts, the model learns to avoid generating the observed incorrect responses, without strengthening the correct responses, thereby reducing training efficiency.

\section{Algorithm}\label{sec:method}
\begin{algorithm*}[t]
\caption{DA3PO Training}
\label{alg:ha3po}
\begin{algorithmic}[1]
\REQUIRE Policy $\pi_\theta$; prompt distribution $\mathcal{D}$; group size $G$; clipping bounds $\epsilon_l,\epsilon_h$; amplification factor $\lambda>1$; difficulty threshold $\tau$
\FOR{each training step}
  \STATE Sample prompts $\{x_j\}$ from $\mathcal{D}$
  \FOR{each prompt $x_j$}
    \STATE Sample $G$ responses $\{o_1,\ldots,o_G\}\sim\pi_{\theta_\text{old}}(\cdot|x_j)$
    \STATE Compute accuracy rewards $\{R_{\mathrm{acc}}(o_1),\ldots,R_{\mathrm{acc}}(o_G)\}$ then final rewards $\{R_1,\ldots,R_G\}$
    \STATE Compute empirical positive rate $\hat{p} \leftarrow \frac{1}{G}\sum_{i=1}^{G} \mathbb{I}\!\left[R_{\mathrm{acc}}(o_i) > 0\right]$
    \IF{$\hat{p}=0$ \textbf{or} $\hat{p}=1$}
      \STATE filter and resample \hfill{\textit{(Dynamic Sampling)}}
    \ENDIF
    \STATE Compute advantages $\hat{A}_{i,t}$ via Eq. (\ref{eq:grpo_advantage})
    \FOR{each response $o_i$}
      \IF{$R_{\mathrm{acc}}(o_i) > 0$ \textbf{and} $\hat{p} < \tau$}
        \STATE $\tilde{A}_{i,t}\leftarrow \lambda\cdot\hat{A}_{i,t}$ \hfill{$\triangleright$ \textit{DAA}}
      \ELSE
        \STATE $\tilde{A}_{i,t}\leftarrow\hat{A}_{i,t}$
      \ENDIF
    \ENDFOR
  \ENDFOR
  \STATE Compute per-token loss:
  \STATE \quad $\mathcal{L}_{i,t}(\theta) \leftarrow\min\!\bigl(r_{i,t}\,\tilde{A}_{i,t},\,\bar{r}_{i,t}\,\tilde{A}_{i,t}\bigr)$
  \STATE Update $\theta$ by maximising $\mathcal{J}_{\text{DA3PO}}(\theta)$ via Eq.~\eqref{eq:da3po}
\ENDFOR
\end{algorithmic}
\end{algorithm*}

As mentioned in Section~\ref{sec:theory}, Dynamic Sampling induces an asymmetric amplification. It causes the model to learn to avoid generating observed incorrect responses, rather than to capitalize on the hard-to-sample correct ones on hard prompts,  which reduces training efficiency. To address this issue, we propose Direct Advantage Amplification (DAA). Its key insight is to amplify the advantage of hard-to-sample correct responses on hard prompts obtained by Dynamic Sampling by a positive constant. This allows such correct responses to be capitalized on more effectively, leading to higher training efficiency. Then, by integrating DAA into DAPO, we obtain \textit{Difficulty-aware Advantage Amplification Policy Optimization} (DA3PO). Its detailed procedure is presented in \cref{alg:ha3po}. More surprisingly, since the implementation of DAA is straightforward, DA3PO requires fewer than 30 lines of code derived from DAPO.

\subsection{Details}
\label{sec:details}

From \cref{thm:asym_amp}, Dynamic Sampling applies asymmetric amplification factors $\alpha^{+}$ and $\alpha^{-}$ to the advantage of correct and incorrect responses, respectively. As shown in the subsequent remark of \cref{thm:asym_amp}, $\alpha^{+}$ remains close to~$1$ on very hard prompts ($p\ll0.5$) while $\alpha^{-}$ is much larger than $\alpha^{+}$. The asymmetric amplification finally reduces training efficiency on hard prompts. 

To address this problem, we can scale the advantage of hard-to-sample correct responses on hard prompts obtained by Dynamic Sampling by a factor of $1/\gamma$. However, $p$ varies as the policy updates during training, making the true value of $\gamma$ difficult to track, so we use a substituting way to improve the training efficiency. We propose Direct Advantage Amplification (DAA), which adopts a constant $\lambda > 1$ to amplify the advantage of hard-to-sample correct responses on hard prompts obtained by Dynamic Sampling.

This design preserves the amplification effect on hard prompts without requiring estimation of the prompt difficulty $p$, while introducing only minimal modifications over DAPO. The calculation formula of DAA is:
\begin{equation}
\label{eq:amplified_advantage}
\tilde{A}_{i,t} =
\begin{cases}
\lambda \cdot \hat{A}_{i,t}, & \text{if } R_i>0 \text{ and } \hat{p}<\tau, \\[3pt]
\hat{A}_{i,t}, & \text{otherwise},
\end{cases}
\end{equation}
where $\lambda>1$ is a constant amplification factor and $\tau\in(0,1)$ a difficulty threshold. DAA is activated only when the response is correct ($R_i>0$) and the prompt is hard ($\hat{p}<\tau$) 
to improve the training efficiency on hard prompts.

Based on experiments with various values of $\lambda$, we set $\lambda = 2$ (see Section~\ref{sec:ablation} for more details).

By integrating DAA into the DAPO objective, we get DA3PO, whose objective is formulated as:

\begin{equation}
\small
\begin{split}
    \mathcal{J}_{\mathrm{DA3PO}}(\theta) = \mathbb{E}_{x \sim \mathcal{D},\; \{o_i\}_{i=1}^G \sim \pi_{\theta_{\mathrm{old}}}(\cdot \mid x)}
    \Bigg[ \frac{1}{\sum_{i=1}^{G}|o_i|} \\
    \sum_{i=1}^{G} \sum_{t=1}^{|o_i|} \mathcal{L}_{i,t}(\theta) \Bigg],
    \quad \text{s.t.} \quad 0 < \hat{p} < 1,
\label{eq:da3po}
\end{split}
\end{equation}
with the token-level loss updated to use the DAA advantage $\tilde{A}_{i,t}$:
\begin{equation}
\begin{split}
    \mathcal{L}_{i,t}(\theta) = \min\!\Big(
    r_{i,t}(\theta)\,\tilde{A}_{i,t},\;
    \mathrm{clip}\!\big(r_{i,t}(\theta),\, \\
    1-\epsilon_l,\, 1+\epsilon_h\big)\,\tilde{A}_{i,t}
    \Big).
\label{eq:per_token_loss}
\end{split}
\end{equation}

Algorithm~\ref{alg:ha3po} summarises the full DA3PO training procedure. Lines marked with $\triangleright$ are the only additions relative to standard DAPO. 

\subsection{Discussion}

To improve the training efficiency of Dynamic Sampling, our DAA amplifies the advantage of hard-to-sample correct responses on hard prompts, as obtained by Dynamic Sampling. In addition, since the implementation of DAA is straightforward, DA3PO requires fewer than 30 lines of code derived from DAPO. Moreover, although we only apply DAA to DAPO, DAA can be applied to GRPO and its variants. This is because we only modify Dynamic Sampling, ensuring compatibility with any other enhancements made to GRPO and its variants. Lastly, our DAA still has significant room for improvement. For example, the constant $\lambda$ and $\tau$ used in Eq. \eqref{eq:amplified_advantage} can be replaced by adaptive factors. We leave such extensions to future work.

\section{Experiments}\label{sec:Experiments}
\subsection{Experimental Setup}
\label{sec:setup}

\paragraph{Models and Training Data.}
We evaluate our DA3PO on two base models of different scales: Qwen3-4B-Base and Qwen3-8B-Base. Following DAPO~\cite{dapo}, we use the DAPO-Math-17K dataset for training, which consists of 17K mathematical problems with verifiable answers.
\vspace{-5pt}
\paragraph{Benchmarks and baselines.}
To evaluate our algorithm, we conduct experiments on seven mathematical benchmarks: AIME24/25/26, AMC, Minerva, Olympiad~\cite{olympiadbench} and MATH~\cite{math}. We compare DA3PO against the base model and three reinforcement learning baselines: GRPO, DAPO, and GSPO.
\vspace{-15pt}
\paragraph{Training.}
We set the rollout group size $G = 16$, training batch size to 512, learning rate to $1 \times 10^{-6}$. The maximum prompt and response lengths are set to 2K and 4K tokens, respectively. As mentioned in Section~\ref{sec:method}, we fix $\tau = 0.5$ for DA3PO. We also use $\lambda = 2$ on both model scales and we discuss the choice of $\lambda$ in Section~\ref{sec:ablation}.
\vspace{-5pt}
\paragraph{Evaluation.}
All algorithms are trained for 200 steps, by which accuracy has stabilized. We select the best-performing checkpoint for each algorithm and report its accuracy. During evaluation, we use temperature as 1 and top-p as 0.95. For the AIME and AMC benchmarks, we sample 32 responses per problem; for MATH, Olympiad, and Minerva, we sample 8 responses per problem. 

\subsection{Main Results}
\label{sec:main-results}

\begin{table*}[t]
\centering
\small
\setlength{\tabcolsep}{4pt}
\renewcommand{\arraystretch}{1.15}
\begin{tabular}{cl *{8}{c}}
\toprule
\multirow{2}{*}{\textbf{Model}} & \multirow{2}{*}{\textbf{Algorithm}}
& \textbf{AIME24} & \textbf{AIME25} & \textbf{AIME26}
& \textbf{AMC}
& \textbf{Minerva} & \textbf{Olympiad}
& \textbf{MATH}
& \multirow{2}{*}{\textbf{AVG}} \\
& & \textbf{avg@32} & \textbf{avg@32} & \textbf{avg@32} & \textbf{avg@32} & \textbf{avg@8} & \textbf{avg@8} & \textbf{avg@8} & \\
\midrule
\multirow{8}{*}{\textbf{Qwen3-4B-Base}}
& Base Model         &  5.21 &  3.23 &  3.65 & 20.18 & 12.22 & 19.56 & 41.23 & 15.04 \\
& GRPO               & 15.31 & 12.81 & 11.56 & 55.50 & 23.99 & 39.35 & 77.22 & 33.68 \\
& DAPO               & 18.96 & 15.10 & 12.71 & 52.07 & 24.17 & 40.65 & 78.20 & 34.55 \\
& GSPO               & 18.85 & 13.17 & 14.27 & 51.51 & 22.61 & 41.96 & 79.87 & 34.61 \\
& \textbf{DA3PO}     & \textbf{23.96} & \textbf{19.17} & \textbf{17.19} & \textbf{60.50} & \textbf{26.84} & \textbf{44.52} & \textbf{85.20} & \textbf{39.63} \\
\cmidrule(l){2-10}
& \textit{vs.\ GRPO} & \impgain{8.65} & \impgain{6.36} & \impgain{5.63} & \impgain{5.00} & \impgain{2.85} & \impgain{5.17} & \impgain{7.98} & \impgain{5.95} \\
& \textit{vs.\ DAPO} & \impgain{5.00} & \impgain{4.07} & \impgain{4.48} & \impgain{8.43} & \impgain{2.67} & \impgain{3.87} & \impgain{7.00} & \impgain{5.08} \\
& \textit{vs.\ GSPO} & \impgain{5.11} & \impgain{6.00} & \impgain{2.92} & \impgain{8.99} & \impgain{4.23} & \impgain{2.56} & \impgain{5.33} & \impgain{5.02} \\
\midrule
\multirow{8}{*}{\textbf{Qwen3-8B-Base}}
& Base Model         &  5.52 &  5.73 &  4.69 & 26.13 & 14.71 & 26.78 & 59.62 & 20.45 \\
& GRPO               & 18.54 & 17.29 & 16.77 & 61.11 & 24.77 & 42.74 & 82.95 & 37.74 \\
& DAPO               & 24.58 & 20.52 & 17.19 & 64.83 & 29.50 & 48.30 & 85.52 & 41.49 \\
& GSPO               & 24.27 & 21.56 & 17.92 & 65.44 & 30.06 & 48.83 & 85.52 & 41.94 \\
& \textbf{DA3PO}     & \textbf{30.52} & \textbf{23.23} & \textbf{24.17} & \textbf{69.69} & \textbf{37.91} & \textbf{54.26} & \textbf{88.45} & \textbf{46.89} \\
\cmidrule(l){2-10}
& \textit{vs.\ GRPO} & \impgain{11.98} & \impgain{5.94} & \impgain{7.40} & \impgain{8.58} & \impgain{13.14} & \impgain{11.52} & \impgain{5.50} & \impgain{9.15} \\
& \textit{vs.\ DAPO} & \impgain{5.94} & \impgain{2.71} & \impgain{6.98} & \impgain{4.86} & \impgain{8.41} & \impgain{5.96} & \impgain{2.93} & \impgain{5.40} \\
& \textit{vs.\ GSPO} & \impgain{6.25} & \impgain{1.67} & \impgain{6.25} & \impgain{4.25} & \impgain{7.85} & \impgain{5.43} & \impgain{2.93} & \impgain{4.95} \\
\bottomrule
\end{tabular}
\caption{Avg@k (\%) of all algorithms across seven mathematical reasoning benchmarks, with absolute improvement of DA3PO over each baseline. We report avg@32 for AIME and AMC, and avg@8 for Minerva, Olympiad and MATH. Best results within each benchmark are \textbf{bolded}.}
\label{tab:main}
\vspace{-0.3cm}
\end{table*}

\begin{table*}[t]
\centering
\small
\setlength{\tabcolsep}{4pt}
\renewcommand{\arraystretch}{1.15}
\begin{tabular}{cl *{8}{c}}
\toprule
\multirow{2}{*}{\textbf{Model}} & \multirow{2}{*}{\textbf{Algorithm}}
& \textbf{AIME24} & \textbf{AIME25} & \textbf{AIME26}
& \textbf{AMC}
& \textbf{Minerva} & \textbf{Olympiad}
& \textbf{MATH}
& \multirow{2}{*}{\textbf{AVG}} \\
& & \textbf{pass@32} & \textbf{pass@32} & \textbf{pass@32} & \textbf{pass@32} & \textbf{pass@8} & \textbf{pass@8} & \textbf{pass@8} & \\
\midrule
\multirow{8}{*}{\textbf{Qwen3-4B-Base}}
& Base Model         & 30.00 & 23.33 & 16.67 & 78.31 & 28.31 & 42.67 & 77.60 & 42.41 \\
& GRPO               & 36.67 & \textbf{46.67} & 30.00 & 86.75 & 36.03 & 56.00 & 90.20 & 54.62 \\
& DAPO               & 43.33 & 43.33 & 30.00 & 84.34 & 39.71 & 59.56 & 91.00 & 55.90 \\
& GSPO               & 50.00 & 40.00 & 30.00 & \textbf{91.57} & 35.66 & 57.33 & 89.20 & 56.25 \\
& \textbf{DA3PO}     & \textbf{60.00} & \textbf{46.67} & \textbf{40.00} & \textbf{91.57} & \textbf{41.18} & \textbf{61.19} & \textbf{93.80} & \textbf{62.06} \\
\cmidrule(l){2-10}
& \textit{vs.\ GRPO} & \impgain{23.33} & 0.00 & \impgain{10.00} & \impgain{4.82} & \impgain{5.15} & \impgain{5.19} & \impgain{3.60} & \impgain{7.44} \\
& \textit{vs.\ DAPO} & \impgain{16.67} & \impgain{3.34} & \impgain{10.00} & \impgain{7.23} & \impgain{1.47} & \impgain{1.63} & \impgain{2.80} & \impgain{6.16} \\
& \textit{vs.\ GSPO} & \impgain{10.00} & \impgain{6.67} & \impgain{10.00} & 0.00 & \impgain{5.52} & \impgain{3.86} & \impgain{4.60} & \impgain{5.81} \\
\midrule
\multirow{8}{*}{\textbf{Qwen3-8B-Base}}
& Base Model         & 23.33 & 33.33 & 23.33 & 83.13 & 31.62 & 46.81 & 85.80 & 46.76 \\
& GRPO               & 43.33 & 40.00 & 43.33 & 90.36 & 38.24 & 61.63 & 93.20 & 58.58 \\
& DAPO               & 56.67 & 46.67 & 46.67 & 90.36 & 43.01 & 63.70 & 93.40 & 62.93 \\
& GSPO               & 50.00 & 40.00 & 46.67 & 89.16 & 45.59 & \textbf{71.41} & \textbf{95.80} & 62.66 \\
& \textbf{DA3PO}     & \textbf{70.00} & \textbf{50.00} & \textbf{53.33} & \textbf{92.77} & \textbf{49.26} & \textbf{71.41} & 94.80 & \textbf{68.80} \\
\cmidrule(l){2-10}
& \textit{vs.\ GRPO} & \impgain{26.67} & \impgain{10.00} & \impgain{10.00} & \impgain{2.41} & \impgain{11.02} & \impgain{9.78} & \impgain{1.60} & \impgain{10.22} \\
& \textit{vs.\ DAPO} & \impgain{13.33} & \impgain{3.33} & \impgain{6.66} & \impgain{2.41} & \impgain{6.25} & \impgain{7.71} & \impgain{1.40} & \impgain{5.87} \\
& \textit{vs.\ GSPO} & \impgain{20.00} & \impgain{10.00} & \impgain{6.66} & \impgain{3.61} & \impgain{3.67} & 0.00 & \imploss{1.00} & \impgain{6.14} \\
\bottomrule
\end{tabular}
\caption{Pass@k (\%) of all algorithms across seven mathematical reasoning benchmarks, with absolute improvement of DA3PO over each baseline. We report pass@32 for AIME and AMC, and pass@8 for Minerva, Olympiad and MATH. Best results within each benchmark are \textbf{bolded}.}
\label{tab:passk}
\vspace{-0.2cm}
\end{table*}

\begin{figure*}[!t]
    \centering
    \begin{minipage}[t]{0.45\linewidth}
        \centering
        \includegraphics[width=\linewidth]{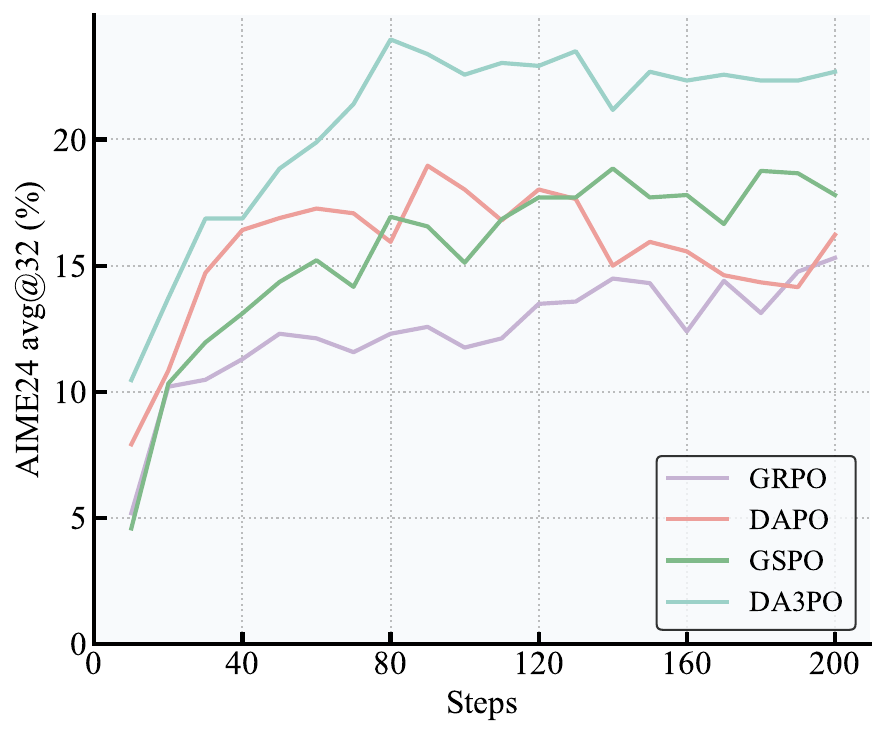}
        \vspace{-0.8cm}
        \subcaption{Avg@32 of AIME24 in training on Qwen3-4B-Base.}
        \label{fig:val_4B}
    \end{minipage}
    \hfill
    \begin{minipage}[t]{0.45\linewidth}
        \centering
        \includegraphics[width=\linewidth]{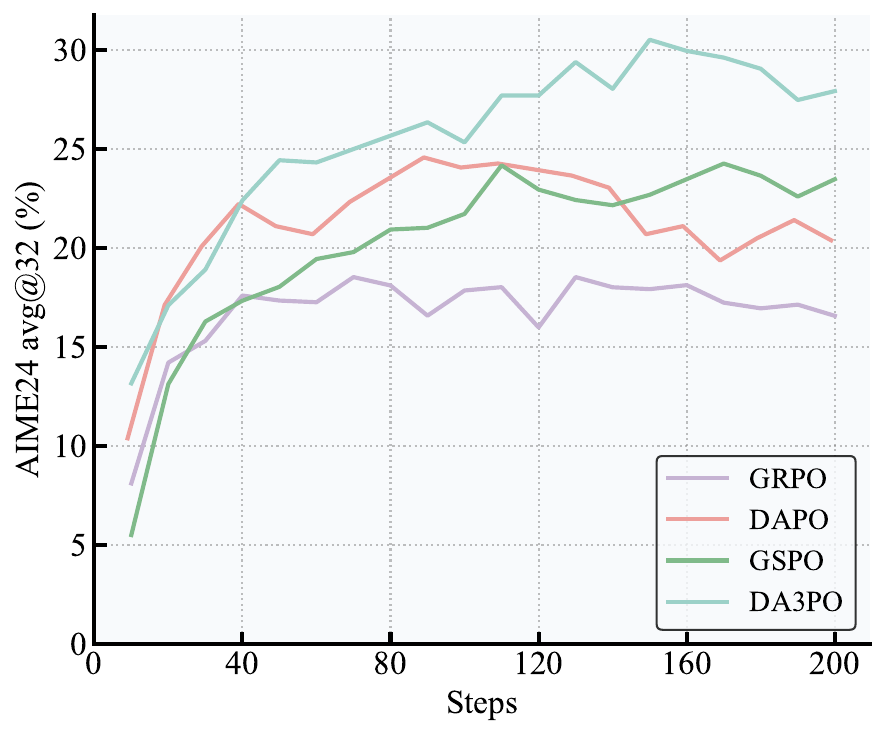}
        \vspace{-0.8cm}
        \subcaption{Avg@32 of AIME24 in training on Qwen3-8B-Base.}
        \label{fig:val_8B}
    \end{minipage}
    \caption{
        Avg@32 of AIME24 over the course of RL training for all algorithms. Due to constraints on computational resources, we only use AIME24 as the validation benchmark.
    }
    \label{fig:val_acc}
\vspace{-0.5cm}
\end{figure*}

\begin{table*}[t]
\centering
\small
\setlength{\tabcolsep}{4pt}
\renewcommand{\arraystretch}{1.15}
\begin{tabular}{cc *{8}{>{\centering\arraybackslash}p{1.2cm}}}
\toprule
\multirow{2}{*}{\textbf{Model}} & \multirow{2}{*}{\textbf{Algorithm}}
& \textbf{AIME24} & \textbf{AIME25} & \textbf{AIME26}
& \textbf{AMC}
& \textbf{Minerva} & \textbf{Olympiad}
& \textbf{MATH}
& \multirow{2}{*}{\textbf{AVG}} \\
& & \textbf{avg@32} & \textbf{avg@32} & \textbf{avg@32} & \textbf{avg@32} & \textbf{avg@8} & \textbf{avg@8} & \textbf{avg@8} & \\
\midrule
\multirow{4}{*}{\textbf{Qwen3-4B-Base}}
& $\lambda\!=\!1.0$  & 18.96 & 15.10 & 12.71 & 52.07 & 24.17 & 40.65 & 78.20 & 34.55 \\
& $\lambda\!=\!1.5$  & 22.08 & \textbf{19.17} & 15.94 & 60.35 & \textbf{26.98} & 44.50 & 83.80 & 38.97 \\
& $\lambda\!=\!2.0$  & \textbf{23.96} & \textbf{19.17} & \textbf{17.19} & \textbf{60.50} & 26.84 & \textbf{44.52} & \textbf{85.20} & \textbf{39.63} \\
& $\lambda\!=\!2.5$  & 23.02 & 15.62 & 13.33 & 58.32 & 25.64 & 39.91 & 78.33 & 36.31 \\
\midrule
\multirow{4}{*}{\textbf{Qwen3-8B-Base}}
& $\lambda\!=\!1.0$  & 24.58 & 20.52 & 17.19 & 64.83 & 29.50 & 48.30 & 85.52 & 41.49 \\
& $\lambda\!=\!1.5$  & 28.65 & 21.56 & 21.35 & 66.72 & 35.57 & 53.19 & 87.02 & 44.87 \\
& $\lambda\!=\!2.0$  & \textbf{30.52} & \textbf{23.23} & 24.17 & 69.69 & \textbf{37.91} & \textbf{54.26} & \textbf{88.45} & \textbf{46.89} \\
& $\lambda\!=\!2.5$  & 29.69 & 21.04 & \textbf{24.69} & \textbf{69.95} & 30.24 & 48.94 & 86.85 & 44.49 \\
\bottomrule
\end{tabular}
\caption{Ablation on the amplification factor $\lambda$. We report avg@32 for AIME and AMC, and avg@8 for Minerva, Olympiad and MATH. $\lambda\!=\!1.0$ is DAPO as a reference. Best results within each benchmark are \textbf{bolded}.}
\label{tab:ablation}
\vspace{-0.5cm}
\end{table*}

\findingbox{DA3PO consistently achieves the highest average accuracy across model scales.}

To verify the effectiveness of our algorithm, we evaluate the avg@k of all algorithms on the seven benchmarks for both model scales. The experimental results are illustrated in Table~\ref{tab:main}.
Firstly, DA3PO outperforms all other GRPO variants by at least 5.02\% in average accuracy on Qwen3-4B-Base and by at least 4.95\% on Qwen3-8B-Base. 
Secondly, DA3PO achieves the best performance on the hardest benchmarks. 

Specifically, DA3PO outperforms baselines by up to 5.00\% and 6.25\% on AIME benchmarks compared to the best baseline for Qwen3-4B-Base and Qwen3-8B-Base, respectively. In fact, most of the improvements across AIME benchmarks are approximately 5\%. Lastly, on easier benchmarks such as AMC, Minerva, Olympiad, and MATH, DA3PO consistently outperforms baselines across both model scales. In particular, even on the easiest MATH benchmark, DA3PO outperforms other variants by 5.33\% on Qwen3-4B-Base and 2.93\% on Qwen3-8B-Base, indicating that DA3PO improves performance on hard prompts without sacrificing performance on easier ones.

\findingbox{DA3PO broadens the coverage of correct responses rather than collapsing onto a single solution mode.}

To verify that DA3PO does not merely raise average accuracy by collapsing onto a single solution mode, we report the pass@k of all algorithms, as shown in Table~\ref{tab:passk}.
Firstly, DA3PO outperforms baselines by at least 5.81\% in average pass@k on Qwen3-4B-Base and by at least 5.87\% on Qwen3-8B-Base. Then, DA3PO exceeds all baselines on five of seven benchmarks on both model scales. For example, DA3PO outperforms baselines by up to 10.00\% for Qwen3-4B-Base and at least 13.33\% for Qwen3-8B-Base on AIME24. On easier benchmarks except MATH, DA3PO consistently outperforms. Besides, on the easiest dataset MATH DA3PO trails by only 1.00\% on Qwen3-8B-Base, as baselines even base model already achieve pass@k above 90\% on MATH, leaving limited room for further performance gains. These results confirm that DA3PO broadens the space of correct responses across the evaluated benchmarks, rather than collapsing the policy onto a single solution mode.

\findingbox{DA3PO stably and reliably reshapes the trajectory of RL training.}
\vspace{5pt}

To verify that DA3PO brings stable performance improvement throughout the training, we plot the avg@32 of AIME24 over training on Qwen3-4B-Base and Qwen3-8B-Base, respectively. The experimental results are shown in Figure~\ref{fig:val_acc}. Interestingly, on both model scales, the DA3PO curve sits visibly above those of GRPO, DAPO, and GSPO throughout the RL training, rather than only at the best checkpoint, confirming that DA3PO stably and reliably reshapes the optimization trajectory of RL training, not an artifact of checkpoint selection or training-step variance.

\subsection{Ablation Study}
\label{sec:ablation}

\findingbox{DA3PO is robust across different values of $\lambda$, with $\lambda = 2$ yielding the best average accuracy on both model scales.}

\vspace{5pt}

To verify the robustness of DA3PO and identify the optimal value of $\lambda$, we conduct an ablation on the amplification factor $\lambda$, the key hyperparameter of DA3PO that controls the magnitude of DAA to correct responses on hard prompts. We vary $\lambda \in \{1.5, 2.0, 2.5\}$ on both model scales. Table~\ref{tab:ablation} reports the avg@k for each setting. Firstly, DA3PO consistently outperforms DAPO ($\lambda = 1.0$) across all tested values of $\lambda$ on both model scales, confirming that the amplification mechanism is robust and effective regardless of the specific choice of $\lambda$. Even  $\lambda = 1.5$ outperforms DAPO by 4.42\% on Qwen3-4B-Base and 3.38\% on Qwen3-8B-Base, confirming that even modest amplification provides a meaningful improvement. Secondly, $\lambda = 2.0$ achieves the best average accuracy on both model scales, as 39.63\% on Qwen3-4B-Base and 46.89\% on Qwen3-8B-Base. Furthermore, the degradation at $\lambda = 2.5$ is more severe on the smaller model, falling 3.32\% below $\lambda = 2.0$, while on Qwen3-8B-Base the decline is smaller and $\lambda = 2.5$ even achieves the highest score on individual benchmarks such as AIME26 and AMC. We attribute this to smaller models generating fewer correct responses on hard prompts, making them more sensitive to overly aggressive amplification, while larger models are more tolerant of higher amplification.

\section{Conclusion}\label{sec:Conclusion}

\vspace{-11pt}

In this paper, we show that Dynamic Sampling induces asymmetrical amplification, which reduces training efficiency in DAPO. To address this issue, we proposed Direct Advantage Amplification (DAA). Its key insight is to amplify the advantage of correct responses on hard prompts so that the hard-to-sample correct responses on hard prompts obtained by Dynamic Sampling can be effectively capitalized on. By integrating DAA into DAPO, we obtain Difficulty-aware Advantage Amplification Policy Optimization (DA3PO), which requires fewer than 30 lines of code derived from DAPO. Experiments on seven mathematical benchmarks demonstrate that DA3PO yields consistent accuracy enhancements across both model scales and all benchmarks, significantly outperforming GRPO and other classical GRPO variants.

\section*{Limitations}\label{sec:Limitations}

While DA3PO effectively addresses the asymmetric advantage amplification induced by Dynamic Sampling, DA3PO adopts a constant $\lambda$ and a fixed difficulty threshold $\tau$ to amplify the advantage of hard-to-sample correct responses on hard prompts. Although this design is straightforward and preserves the amplification effect on hard prompts, it may not optimally compensate the asymmetric amplification across all prompt difficulties throughout training. Therefore, replacing the constant $\lambda$ and $\tau$ with adaptive factors that dynamically adjust to the evolving prompt difficulty remains a promising direction for future work.

\bibliography{custom}

\clearpage
\newpage

\appendix

\section{Related Work}\label{sec:Related Work}

As the original group relative reinforcement learning algorithm, GRPO~\cite{grpo} has attracted considerable research attention. However, instabilities in GRPO accumulate over long-term training, preventing the policy from consistently reinforcing correct responses~\cite{dr-grpo,gpg,dapo}. To address this, several works~\cite{dr-grpo,gpg,reinforce++,gspo,dapo,dcpo,gmpo,gdpo,fipo,yuchen2026predictability} focus on improving the objective function itself to enhance training stability. Among them, DAPO~\cite{dapo} is one of the most representative works, introducing four improvements over GRPO, including Dynamic Sampling, and achieving higher accuracy than GRPO.

Building upon these stability improvements, several works further aim to improve training efficiency in LLMs reinforcement learning. These works~\cite{entropyperspective,entropy-guided,80/20,enhancing,sent} design algorithms to explore GRPO that regulate entropy dynamics, operating on the clipping, token masking, temperature scheduling, advantage reweighting, etc. However, these works largely overlook the impact of Dynamic Sampling on training efficiency, which serves as also a critical improvement in DAPO.

In this work, we reveal that Dynamic Sampling, while enhancing accuracy through improved training stability, actually reduces training efficiency. Building on this analysis, we propose DA3PO, which improves the training efficiency of Dynamic Sampling while preserving the original training stability benefits and achieves significant improvements in accuracy across both model scales and all benchmarks.

\section{Proof of Theorem~\ref{thm:asym_amp}}
\label{app:proof_thm1}

\subsection{Proof of \texorpdfstring{Eq.~\eqref{eq:amp_correct}}{Eq. (X)}}

\begin{proof}
Since $R_{\mathrm{acc}}(o_i){=}1$ guarantees that at least one of the $G$ responses is correct, the event $\hat{p}=0$ (all responses incorrect) is impossible. Therefore, the sample space conditioned on $R_{\mathrm{acc}}(o_i){=}1$ is partitioned into two mutually exclusive cases: (i)~$\hat{p}=1$, where all $G$ responses are correct, and (ii)~$0<\hat{p}<1$, where at least one response is incorrect.

We first compute $P(\hat{p}=1 \mid R_{\mathrm{acc}}(o_i){=}1)$. Since $R_{\mathrm{acc}}(o_i){=}1$ already fixes one response as correct, by the independence of responses, the remaining $G{-}1$ responses must each be independently correct. It follows that
\begin{equation}\label{eq:thm1_allcorrect}
    P(\hat{p}=1 \mid R_{\mathrm{acc}}(o_i){=}1) = p^{G-1},
\end{equation}
which implies
\begin{equation}\label{eq:thm1_nondegen_correct}
    P(0<\hat{p}<1 \mid R_{\mathrm{acc}}(o_i){=}1) = 1 - p^{G-1}.
\end{equation}

When $\hat{p}=1$, all $G$ responses are correct and hence all rewards are identical. As a result, $z$-score normalisation yields zero variance, and we obtain
\begin{equation}\label{eq:thm1_collapse_correct}
    \mathbb{E}[\hat{A}_i \mid \hat{p}=1,\, R_{\mathrm{acc}}(o_i){=}1] = 0.
\end{equation}

By the law of total expectation over the partition established above, we have
\begin{equation}\label{eq:thm1_total_exp_correct}
    \begin{aligned}
        & \mathbb{E}[\hat{A}_i \mid R_{\mathrm{acc}}(o_i){=}1] \\
        = {} & P(\hat{p}=1 \mid R_{\mathrm{acc}}(o_i){=}1) \\
             & \quad \cdot \mathbb{E}[\hat{A}_i \mid \hat{p}=1,\, R_{\mathrm{acc}}(o_i){=}1] \\
             & + P(0<\hat{p}<1 \mid R_{\mathrm{acc}}(o_i){=}1) \\
             & \quad \cdot \mathbb{E}[\hat{A}_i \mid 0<\hat{p}<1,\, R_{\mathrm{acc}}(o_i){=}1].
    \end{aligned}
\end{equation}
By substituting Eq.~\eqref{eq:thm1_allcorrect}, \eqref{eq:thm1_nondegen_correct}, and \eqref{eq:thm1_collapse_correct} into Eq.~\eqref{eq:thm1_total_exp_correct}, we get
\begin{equation}\label{eq:thm1_simplified_correct}
    \begin{aligned}
        & \mathbb{E}[\hat{A}_i \mid R_{\mathrm{acc}}(o_i){=}1] \\
        = {} & (1-p^{G-1}) \\
             & \quad \cdot \mathbb{E}[\hat{A}_i \mid 0<\hat{p}<1,\, R_{\mathrm{acc}}(o_i){=}1].
    \end{aligned}
\end{equation}
Since $1 - p^{G-1} > 0$ for $p \in (0,1)$, rearranging Eq.~\eqref{eq:thm1_simplified_correct} yields
\begin{equation}
    \begin{aligned}
        & \mathbb{E}[\hat{A}_i \mid 0<\hat{p}<1,\, R_{\mathrm{acc}}(o_i){=}1] \\
        & \quad = \frac{\mathbb{E}[\hat{A}_i \mid R_{\mathrm{acc}}(o_i){=}1]}{1-p^{G-1}} \\
        & \quad = \alpha^{+}\cdot\mathbb{E}[\hat{A}_i \mid R_{\mathrm{acc}}(o_i){=}1],
    \end{aligned}
\end{equation}
which is Eq.~\eqref{eq:amp_correct}. We complete the proof.
\end{proof}

\subsection{Proof of \texorpdfstring{Eq.~\eqref{eq:amp_incorrect}}{Eq. (X)}}

\begin{proof}
Since $R_{\mathrm{acc}}(o_i){=}{-1}$ guarantees that at least one of the $G$ responses is incorrect, the event $\hat{p}=1$ (all responses correct) is impossible. Therefore, the sample space conditioned on $R_{\mathrm{acc}}(o_i){=}{-1}$ is partitioned into two mutually exclusive cases: (i)~$\hat{p}=0$, where all $G$ responses are incorrect, and (ii)~$0<\hat{p}<1$, where at least one response is correct.

We first compute $P(\hat{p}=0 \mid R_{\mathrm{acc}}(o_i){=}{-1})$. Since $R_{\mathrm{acc}}(o_i){=}{-1}$ already fixes one response as incorrect, by the independence of responses, the remaining $G{-}1$ responses must each be independently incorrect. It follows that
\begin{equation}\label{eq:thm1_allincorrect}
\begin{aligned}
    & P(\hat{p}=0 \mid R_{\mathrm{acc}}(o_i){=}{-1}) \\ 
    = & (1-p)^{G-1},
\end{aligned}
\end{equation}
which implies
\begin{equation}\label{eq:thm1_nondegen_incorrect}
\begin{aligned}
    & P(0<\hat{p}<1 \mid R_{\mathrm{acc}}(o_i){=}{-1}) \\ = & 1 - (1-p)^{G-1}.
\end{aligned}
\end{equation}

When $\hat{p}=0$, all $G$ responses are incorrect and hence all rewards are identical. As a result, $z$-score normalisation yields zero variance, and we obtain
\begin{equation}\label{eq:thm1_collapse_incorrect}
\begin{aligned}
    \mathbb{E}[\hat{A}_i \mid \hat{p} =0,\, R_{\mathrm{acc}}(o_i){=}{-1}] = 0.
\end{aligned}
\end{equation}

By the law of total expectation over the partition established above, we have
\begin{equation}\label{eq:thm1_total_exp_incorrect}
    \begin{aligned}
        & \mathbb{E}[\hat{A}_i \mid R_{\mathrm{acc}}(o_i){=}{-1}] \\
        = {} & P(\hat{p}=0 \mid R_{\mathrm{acc}}(o_i){=}{-1}) \\
             & \quad \cdot \mathbb{E}[\hat{A}_i \mid \hat{p}=0,\, R_{\mathrm{acc}}(o_i){=}{-1}] \\
             & + P(0<\hat{p}<1 \mid R_{\mathrm{acc}}(o_i){=}{-1}) \\
             & \quad \cdot \mathbb{E}[\hat{A}_i \mid 0<\hat{p}<1,\, R_{\mathrm{acc}}(o_i){=}{-1}].
    \end{aligned}
\end{equation}
By substituting Eq.~\eqref{eq:thm1_allincorrect}, \eqref{eq:thm1_nondegen_incorrect}, and \eqref{eq:thm1_collapse_incorrect} into Eq.~\eqref{eq:thm1_total_exp_incorrect}, we get
\begin{equation}\label{eq:thm1_simplified_incorrect}
    \begin{aligned}
        & \mathbb{E}[\hat{A}_i \mid R_{\mathrm{acc}}(o_i){=}{-1}] \\
        = {} & (1-(1-p)^{G-1}) \\
             & \quad \cdot \mathbb{E}[\hat{A}_i \mid 0<\hat{p}<1,\, R_{\mathrm{acc}}(o_i){=}{-1}].
    \end{aligned}
\end{equation}
Since $1 - (1-p)^{G-1} > 0$ for $p \in (0,1)$, rearranging Eq.~\eqref{eq:thm1_simplified_incorrect} yields
\begin{equation}
    \begin{aligned}
        & \mathbb{E}[\hat{A}_i \mid 0<\hat{p}<1,\, R_{\mathrm{acc}}(o_i){=}{-1}] \\
        & \quad = \frac{\mathbb{E}[\hat{A}_i \mid R_{\mathrm{acc}}(o_i){=}{-1}]}{1-(1-p)^{G-1}} \\
        & \quad = \alpha^{-}\cdot\mathbb{E}[\hat{A}_i \mid R_{\mathrm{acc}}(o_i){=}{-1}],
    \end{aligned}
\end{equation}
which is Eq.~\eqref{eq:amp_incorrect}. We complete the proof.
\end{proof}

\end{document}